\documentclass{article}

\usepackage[utf8]{inputenc} 
\usepackage[T1]{fontenc}    
\usepackage{hyperref}       
\usepackage{url}            
\usepackage{booktabs}       
\usepackage{amsfonts}       
\usepackage{nicefrac}       
\usepackage{microtype}      
\usepackage{graphicx}
\usepackage{amsmath}
\usepackage{bm}
\usepackage{natbib}
\usepackage{amsthm}
\usepackage{fullpage}
\usepackage{mathtools}
\usepackage[mathscr]{euscript}
\usepackage{bbm}
\usepackage{dsfont}
\usepackage{overpic}
\usepackage{enumitem}
\setlist[itemize]{leftmargin=*}
\usepackage{algorithm}
\usepackage{subcaption}
\usepackage{caption}
\usepackage{algorithmic}
\usepackage{color-edits}
\addauthor{matt}{blue}

\DeclareMathOperator*{\argmax}{arg\,max}

\newtheorem{theorem}{Theorem}

\newtheorem{proposition}{Proposition}

\newtheorem{remark}{Remark}

\usepackage{hyperref}

\hypersetup{
  colorlinks   = true,
  urlcolor     = blue,
  linkcolor    = blue,
  citecolor   = blue
}

\title{Evaluating Fairness of Machine Learning Models Under Uncertain and Incomplete Information}

\author{Pranjal Awasthi\\Google \& Rutgers University\\pranjalawasthi@google.com
 \and
Alex Beutel\\Google\\alexbeutel@google.com
\and 
Matthäus Kleindessner\\Amazon\thanks{Work done while being a postdoc at the University of Washington.}\\matkle@amazon.com
\and 
Jamie Morgenstern\\University of Washington \& Google\\jamiemmt@cs.washington.edu
\and
Xuezhi Wang\\Google\\xuezhiw@google.com}

\begin{document}
\maketitle
\date{}
\begin{abstract}
Training and evaluation of fair classifiers is a challenging problem. This is partly due to the fact that most fairness metrics of interest depend on both the sensitive attribute information and label information of the data points. In many scenarios it is not possible to collect large datasets with such information. An alternate approach that is commonly used is to separately train an attribute classifier on data with sensitive attribute information, and then use it later in the ML pipeline to evaluate the bias of a given classifier. While such decoupling helps alleviate the problem of {\em demographic scarcity}, it raises several natural questions such as: {\em how should the attribute classifier be trained?}, and {\em how should one use a given attribute classifier for accurate bias estimation?} In this work we study this question from both theoretical and empirical perspectives. 

We first experimentally demonstrate that the test accuracy of the attribute classifier is not always correlated with its effectiveness in bias estimation for a downstream model. In order to further investigate this phenomenon, we analyze an idealized theoretical model and characterize the structure of the optimal classifier. Our analysis has surprising and counter-intuitive implications where in certain regimes one might want to distribute the error of the attribute classifier as unevenly as possible among the different subgroups. Based on our analysis we develop heuristics for both training and using attribute classifiers for bias estimation in the data scarce regime. We empirically demonstrate the effectiveness of our approach on real and simulated data.   
\end{abstract}

\maketitle

\section{Introduction}
\label{sec:intro}
Estimating a machine learning~(ML) model's difference in accuracy for different  demographic groups is a challenging task in many applications, in large part due to imperfect or missing demographic data~\cite{holstein2019improving}. Demographic information might be unavailable 
for many reasons: users may choose to withhold this information (for privacy or other personal preferences), or the dataset may have been constructed for use in a setting where collecting the  demographic information of the dataset's subjects was unnecessary, undesirable, or even illegal \cite{zhang2018assessing, weissman2011advancing, fremont2016race}.

Responsible ML practitioners and auditors may want to evaluate model performance across different demographic groups even if demographic information was not collected explicitly, and as researchers of fairness we want to enable this.
For example, a microlender who builds a model to prescreen applicants for a loan may explicitly avoid maintaining demographic records on applicants, but may also want to ensure that the rate at which their model approves a candidate is nearly independent of the applicant's gender.  Can one assess this bias of the model, even when  training and validation data do not contain explicit gender information? 

This presents a common problem: how can one audit the performance of a model for predicting $Y$, \emph{conditioned on demographic information $A$},
 when $A$ and $Y$ are observed jointly on little or no data? More formally, consider  two datasets $D_1, D_2$, where $D_1$ has features in $X$ and labels in $Y$, and $D_2$ has  features in $X$ and demographic information in $A$. How should we use these datasets to evaluate a model $f$'s performance predicting $Y$ conditioned on $A$, for the distribution that generated $D_1$? 
 With no further assumptions, 
  this is impossible: if the two datasets are generated from distributions bearing no resemblance to each other, correlations existing between $X$ and $Y$ (and between $X$ and $A$) may be arbitrarily different for the two distributions. 
Modelers may still wish to understand if their separate datasets can help them with this estimation problem, and what properties their two datasets might need to have to result in good estimations. 

One approach, particularly natural for enabling many different teams participating in the development of models without demographic information, 
 is to use proxy attributes as a replacement for the true sensitive attribute. The use of such proxy attributes has been widely used in domains such as health care \cite{brown2016using, fremont2005use} and finance \cite{bureau2014using}. One can view a particular proxy attribute as a coarse {\em attribute classifier} that predicts the sensitive attribute information for a given example. For popular fairness metrics such as equalized odds or demographic parity \cite{calders2009building, zliobaite2015relation,hardt2016equality,zafar2017www}, the sensitive attribute information obtained from the classifier can then be used in conjunction with the predicted label information to get an estimate of the bias of a given model. The advantage of this approach is that it decouples the data requirement, and in particular, the attribute classifier can be trained on a separate dataset without any label information present. On the other hand, such an approach may produce very poor estimates, depending upon how the sources of data relate to one another. 
 
 Nonetheless, it is interesting to ask what principles should guide the choice of a proxy attribute, or more generally the design of a sensitive attribute classifier, and similarly, what principles should be followed when using such classifiers for bias estimation. In this work, we aim to understand this question.
 If one wishes to produce the most accurate estimate of the bias term, is the optimal approach to train the attribute classifier to have the highest possible $0/1$ accuracy, or
 to have similar error rates for each demographic group? 
It turns out that neither of these criteria alone captures the structure of the {\em optimal} attribute classifier.

To demonstrate this, we consider the measurement of violation of {\em equal opportunity}, or true positive rates for different groups \cite{hardt2016equality}, and show that the error in the bias estimates, as a result of using a proxy attribute/attribute classifier depends
not only on the error rate of the classifiers but also on how these errors (and the data more generally) are distributed.
Our analysis has surprising and counter-intuitive implications -- in certain regimes 
the optimal attribute classifier is the one that has the most unequal distribution of errors across the two subgroups! 
Our main contributions are summarized below.
 \begin{itemize}
     \item \textbf{Problem formulation:} We formalize and study the problem of estimating the bias of a given machine learning model using data with only label information and model output, and using a separate attribute classifier to predict the sensitive attributes. We call this the {\em demographically scarce} regime. In particular, our setting captures the common practice of using proxy attributes as a replacement for the true sensitive attribute. We experimentally demonstrate that the accuracy of the attribute classifier does not solely capture its effectiveness in bias estimation.
     \item \textbf{Theoretical analysis:} For the case of the {\em equal opportunity} metric \cite{hardt2016equality}, we present a theoretical analysis under a simplified model. Our analysis completely characterizes the structure of the optimal attribute classifier and shows that the error in the bias estimates depends on the error rate of the classifier as well as the {\em base error rate}, a property of the data distribution. As a result, we show that in certain regimes of the base error rate, the optimal attribute classifier, surprisingly, is the one that has the most unequal distribution of errors among the two subgroups.
     \item \textbf{Empirical validation:} Our theory predicts natural interventions, aimed at better bias estimation, that can be performed either at the stage of designing the attribute classifier or at the stage of using a given attribute classifier. We study their effectiveness in scenarios where our modeling assumptions hold and their sensitivity to violations of the assumptions.
     \item \textbf{An active sampling algorithm:} Finally, we propose an active sampling based algorithm for accurately estimating the bias of a model given access to an attribute classifier for predicting the sensitive attribute. Our approach makes no assumption about the structure of the attribute classifier or the data distribution. We demonstrate via experiments on real data that our sampling scheme performs accurate bias estimation using significantly fewer samples in the demographically scarce regime, as compared to other approaches.

 \end{itemize}
 While highlighting the need for careful considerations when designing such attribute classifiers, our analysis provides concrete recommendations and algorithms that can be used by 
 ML 
 practitioners.

 \begin{remark}
Inferring demographic information should be done with care, even when used only 
for auditing a model's disparate treatment. Any estimation of demographics will invariably have inaccuracies.
This can be particularly problematic when considering, for example, gender, as misgendering an individual can cause them substantial harm. 
Furthermore, building a model to infer demographic information purely for internal evaluation leaves open the possibility that the model may ultimately be used more broadly, 
with possibly unintended consequences. 

That said, preliminary evidence of bias may be invaluable in arguing for resources
to do a more careful evaluation of a machine learning system. For this reason, we believe 
that 
methods to estimate the difference in predictive performance across demographic groups can serve as a useful first step in understanding and improving the system's treatment of underserved populations.
\end{remark}

\section{Related Work}


Various fairness metrics of interest have been proposed for classification~\cite{hardt2016equality, dwork2012fairness, pleiss2017fairness, kleinberg2017}, ranking~\cite{celis_fair_ranking, beutel2019fairness}, unsupervised learning~\cite{chierichetti2017fair,samadi2018price, fair_k_center_2019}, and online learning~\cite{joseph2016fairness,jabbari2017fairness,gillen2018online, blum2018preserving}. 
%
%
These notions of fairness are usually categorized as either  {\em group fairness} or {\em individual fairness}. In the case of group fairness, training a model to satisfy such fairness metrics and evaluating the metrics on a pre-trained model typically requires access to both the sensitive attribute and the label information on the same dataset. 
Only recently, ML research
has explored the implications of using proxy attributes and the sensitivity of fair machine learning algorithms to this noisy demographic data \cite{gupta2018,lamy2019, awasthi2019equalized, coston2019, schumann2019transfer}.

The most related work to ours \cite{chen2019fairness,kallus2020} also study 
the problem of estimating the bias of predictor $f$ of $Y$ from $X$ as a function of $A$ with few or no samples from the joint distribution over $(X, Y, A)$.
In \cite{chen2019fairness} the authors aim to estimate the disparity of a model $f$ predicting $Y$ from $X$, along $A$ by considering predicting $A$ from $X$ via a threshold $q$, and then assuming $A=a$ whenever $P[A = a | X  = x] > q$. This can  work well with unlimited data from the joint distribution over $A, X$, but can both overestimate and underestimate the bias because of the misclassification of $A$ using said thresholding rules. 

The work of \cite{kallus2020} investigates how to perfectly or partially identify the same bias parameter, making few or no assumptions about independence between $A, Y, f$ and $X$. Given that exact identification is usually impossible, the authors describe how to find feasible regions of various disparity measures making very few assumptions about independence between $f, Y, X, A$. Similar to our work, the authors in \cite{kallus2020} study the setting where one has access to a dataset involving $(y,x)$ and another dataset involving $(a,x)$. Given that exact identification is usually impossible, they construct an uncertainty set (confidence interval) for bias estimates assuming access to the Bayes optimal classifier, i.e., the classifier that predicts sensitive attribute according to $P(a|x)$. An important contribution of our paper is to show that the Bayes optimal attribute classifier is not the right classifier in the first place. We demonstrate this via experiments in Table~\ref{tab:model_selection}, Theorem~\ref{thm:bayes-opt-example}, and in fact we go deeper and characterize the structure of the optimal attribute classifier in Theorem~\ref{thm:opt-classifier}~(under assumptions). 

Our work also explains why the confidence intervals in the experiments of~\cite{kallus2020} are large. In particular, the work of~\cite{kallus2020} focuses only on the setting where there is no common data, hence their confidence regions will never shrink with dataset size due to inherent uncertainty. We however believe that it is practically more relevant to study how we can use a little amount of common data and produce accurate estimates. This motivates our active sampling approach as proposed in Section~\ref{sec:active_sampling}.


\section{Setup and Notation}


We consider the classification setting where $\mathcal{X}$ denotes the feature space, $\mathcal{Y}$  the label space and $\mathcal{A}$  the sensitive attribute space. For simplicity, we  assume  $\mathcal{Y} = \{-1,+1\}$ and $\mathcal{A} = \{0,1\}$. Consider two datasets $D_1$ and $D_2$. $D_1$ is drawn from the marginal over $(\mathcal{X}, \mathcal{A})$, of a distribution $P$ over $\mathcal{X} \times \mathcal{Y} \times \mathcal{A}$. Similarly, $D_2$ is drawn from the marginal over $(\mathcal{X}, \mathcal{Y})$, of a distribution $Q$ over $\mathcal{X} \times \mathcal{Y} \times \mathcal{A}$. Thus,  both datasets contain the same collection of features, $D_1$ contains attribute information, and $D_2$ contains label information. In this setting, 
given a classifier $f: \mathcal{X} \rightarrow \mathcal{Y}$, we aim to estimate the {\em bias} of $f$: the difference in the performance of $f$ on $Q$ over $A$. In this work we define the bias to be the {\em equal opportunity} metric \cite{hardt2016equality} defined as
\begin{align}
\label{eq:equal-opportunity}
bias_Q(f) &= |\alpha - \beta|, 
\end{align}
where
\begin{align}\label{eq:def_true_alpha}
   \alpha = \mathbb{P}_{(x,a,y) \sim Q} \big(f(x)=1 | y=1, a=1 \big),\\
   \beta = \mathbb{P}_{(x,a,y) \sim Q} \big(f(x)=1 | y=1, a=0 \big). 
\end{align}
We consider the {\em demographically scarce} regime where the designer of an attribute classifier has access to $D_1$ and the user of the label classifier has access to $D_2$ and additionally has access to the attribute classifier $h: \mathcal{X} \rightarrow \mathcal{A}$ trained on $D_1$. Here, no party has access to both datasets. We aim to  understand how the classifiers $(h,f)$ can be used in combination for accurate bias estimation.

From the perspective of the user of an attribute classifier, it is easy to see that no matter what the attribute classifier is, one cannot get any non-trivial estimate of the bias of $f$ without access to some common data drawn from $\mathcal{X} \times \mathcal{Y} \times \mathcal{A}$. Here we provide a simple proof of this fact. In particular, we will construct a scenario where there is a fixed attribute classifier $h: \mathcal{X} \to \mathcal{A}$. We will construct two distributions $Q_1, Q_2$ such that the output distribution of the attribute classifier is the same in both the cases. Furthermore, any label classifier trained on just $(x,y)$ pairs will have the same output distribution in both the cases. However, the estimates of the bias when using the attribute classifier will differ wildly. Formally, we have the following proposition.
\begin{proposition}
\label{prop:lower-bound-stronger}
Consider a fixed attribute classifier $h: \mathcal{X} \to \mathcal{A}$ and a fixed learning algorithm ALG for computing a label classifier from a distribution over $\mathcal{X} \times \mathcal{Y}$. Then, there exists two distribution $Q_1, Q_2$ over $\mathcal{X} \times \mathcal{Y} \times \mathcal{A}$ such that: a) $Q_1, Q_2$ have the same marginals over $\mathcal{X} \times \mathcal{A}$, b) ALG produces the same label classifier $f$ given access to $Q_1$ or $Q_2$, and c) either the label classifier $f$ is the Bayes optimal classifier, or the ratio of the estimated bias~(using $h$) to the true bias of $f$ on either $Q_1$ or $Q_2$ is $0$ or $\infty$. 
\end{proposition}
\begin{proof}
Given the learning algorithm ALG for the labels and an attribute classifier $h$, we will construct two distributions $Q_1$ and $Q_2$ over $\mathcal{X} \times \mathcal{Y} \times \mathcal{A}$ such that the following holds:
\begin{itemize}
\item $Q_1(\mathcal{X})$ is the same as ${Q_2}(\mathcal{X})$, thereby ensuring that the distribution of $h(x)$ over $Q_1, Q_2$ is the same. 
\item ${Q_1}(\mathcal{X} \times \mathcal{Y})$ is the same as ${Q_2}(\mathcal{X} \times \mathcal{Y})$, thereby ensuring that ALG produces the same classifier $f$.
\item On $Q_1$ the true bias of $f$ is zero and on $Q_2$ the true bias of $f$ is one, thereby ensuring that the ratio of the estimated will be distorted by either zero or $\infty$.
\end{itemize}
To construct the two distributions fix any marginal distribution over $\mathcal{X} \times \mathcal{Y}$ and let $f$ be any classifier for predicting $y$ given $x$ that is not Bayes optimal. Since $f$ is not Bayes optimal, we can consider two non-empty regions $A=\{f=1, Y=1\}$ and $B=\{f=0, Y=1\}$. Construct $Q_1$ by setting the conditional distribution (conditioned on $A$) of $\mathcal{A}$ to be uniformly distributed and similarly the distribution of $\mathcal{A}$ conditioned on $B$ to be uniformly distributed. Construct $Q_2$ to be such that the conditional distribution of $\mathcal{A}$ on $A$ is entirely supported on $a=1$ and the conditional distribution of $\mathcal{A}$ on $B$ is entirely supported on $a=0$. Now it is easy to see that $f$ has zero bias on $Q_1$ and a bias value of one on $Q_2$. Since the estimated bias using $h$ will result in the same value for $Q_1$ and $Q_2$, in at least one of the cases the distortion must be $0$ or $\infty$.
\end{proof}

In the next two sections we study the following questions: a) how should one design an attribute classifier?, and b) how can one use an attribute classifier along with minimal common data and perform accurate bias estimation? We will focus on the case of $P=Q$, i.e., both the designer and the user of an attribute classifier have marginals from the same joint distribution. 


\section{Perspective of the Attribute Classifier Modeler}

What properties of the attribute classifier are desirable for bias estimation? At first thought, training the most accurate classifier seems to be a reasonable strategy. However, it is easy to construct instances where using the Bayes optimal attribute classifier does not lead to the most accurate bias estimation. This is formalized in the Theorem below. See Appendix for the proof.
\begin{theorem}
\label{thm:bayes-opt-example}
Let $x_1, x_2$ be binary attributes, $a$ be a binary sensitive attribute and $y$ be a binary class label. Furthermore, let $f$ be a label classifier given by $x_2$, i.e., $f(x_1, x_2, a) = x_2$. Then there exists a joint distribution $Q$ over $(x_1, x_2, a, y)$ tuples such that the bias of $f$ is zero, i.e., $bias_Q(f) = 0$. On the other hand, using the Bayes optimal predictor for the attribute $a$, leads to an estimated bias of one for $f$.
\end{theorem}

Next, we show that even in practical settings, there is little correlation between the accuracy of the attribute classifier and its effectiveness for bias estimation. Table~\ref{tab:model_selection} contains the result of various attribute classifiers trained on the UCI Adult dataset \cite{Dua:2019}. As can be seen, the accuracy of an attribute classifier is not necessarily correlated with its ability to estimate the true bias of a model. 

\begin{table}[t]
\centering
\begin{tabular}{|c|c|c|c|}
\hline
    Attribute Classifier & Test Accuracy  & $\hat{\alpha}-\hat{\beta}$ & $\alpha-\beta$\\
     \hline
     \hline
    Random Forest & 84.46\% &  -0.1243 &  -0.1149 \\
    \hline
    \hline
    Logistic Regression (LR) & 84.29\% &  -0.1450 &  -0.1149\\
    \hline
    LR (M+O+R) & 82.68\% &  -0.1200 & -0.1149\\
    \hline
    LR (W+M+O+R) & 83.43\% & -0.1454  & -0.1149\\
    \hline
    LR (W+E+M+O+R) & 83.43\% &  -0.1352 &   -0.1149\\
    \hline
    \hline
    SVM & 83.80\% &  -0.2045 &  -0.1149\\
    \hline
    \hline
    1-hidden-layer NN & 83.35\% &  -0.1025 &  -0.1149\\
    \hline
\end{tabular}
\vspace{0.05in}
\caption{Distribution of test accuracies of different attribute (Gender) classifiers, and the resulting estimated biases, over the UCI Adult dataset. $\hat{\alpha}-\hat{\beta}$ is the estimated bias using the attribute classifier and $\alpha - \beta$ is the true bias~(without the absolute value).
The label classifier is a fixed Random Forest Classifier to predict income, with $85.79\%$ test accuracy. 
By default all models are trained with all features except Gender and Income. For the Logistic Regression (LR) model, the names within the bracket indicate which features are used: R: ``Relationship'', E: ``Education", M: ``MaritalStatus", O: ``Occupation", W: ``WorkClass".}
\label{tab:model_selection}
\vspace{-0.2in}
\end{table}

In order to gain further understanding, we consider a simplified model where we aim to theoretically characterize the structure of the optimal attribute classifier. In particular, we will consider the setting where the two datasets $D_1$ and $D_2$ are drawn from distributions which are marginals of the same distribution. For an attribute classifier $h: \mathcal{X} \rightarrow \mathcal{A}$, define the conditional errors of $h$ with respect to $a$ as  
\begin{align}\label{eq:cond-error-def}
g_1 &= \mathbb{P}(h(x) \neq a | a=0, y=1),\\
g_2 &= \mathbb{P}(h(x) \neq a | a=1, y=1).
\end{align}

Here $(x,a)$ is drawn from the marginal of joint distribution over $\mathcal{X}  \times \mathcal{Y} \times \mathcal{A}$ from which both $D_1$ and $D_2$ are drawn. Suppose we use $h$ to predict the noisy attributes $\hat{a}$ on $D_2$, and then use these estimates to measure the bias of a label classifier $f$. To simplify our analysis we make the following key assumption. Below we denote $\hat{y}$ to represent the prediction of the label classifier $f$.

\noindent \textbf{Assumption I:} Given $y$ and $a$, $\hat{y}$ and $\hat{a}$ are conditionally independent, 
that is for all $y_1,y_2\in\{-1,+1\}$ and $a_1,a_2\in\{0,1\}$, we have
\begin{align}
    \label{eq:cond_ind_assumption}
    &\mathbb{P}(\hat{y} = y_1,\hat{a}=a_1| y=y_2, a=a_2) \nonumber\\  &=\mathbb{P}(\hat{y} = y_1| y=y_2, a=a_2)\cdot \mathbb{P}(\hat{a}=a_1| y=y_2, a=a_2).
\end{align}
Notice that while the above assumption is restrictive, there are practical scenarios where one can expect it to hold. Furthermore, the assumption has been studied in prior works to understand noise sensitivity of post-processing and in-processing methods for building fair classifiers \cite{lamy2019,awasthi2019equalized}. Two settings where such an assumption would hold are as follows. First, this would hold if the attribute classifier uses a set of features that are conditionally independent of the features used by the label classifier. The  assumption will also hold  if the attribute classifier makes independent errors with a certain probability. This, for instance, might be the case for certain classifiers based upon crowdsourcing. As we will soon see, even under the above simplified setting, the structure of the optimal attribute classifier can defy convention wisdom. Using the attribute classifier $h$, we get estimates $\hat{\alpha}, \hat{\beta}$ of the true values $\alpha, \beta$~(see Eq.~\eqref{eq:def_true_alpha}), where 
\begin{align*}
   \hat{\alpha} &= \mathbb{P} (f(x)=1 | y=1, \hat{a}=1),\\
   \hat{\beta} &= \mathbb{P} (f(x)=1 | y=1, \hat{a}=0).
\end{align*}
Under Assumption I, we have the following relationship between the true and the estimated values 
\begin{align}
\label{eq:noisy-estimates}
\begin{split}
    \hat{\alpha} &=
    \frac{\alpha \mathbb{P}(y=1,a=1)(1-g_2) + \beta \mathbb{P}(y=1,a=0)g_1}{\mathbb{P}(y=1,a=1)(1-g_2) + \mathbb{P}(y=1,a=0)g_1},  \\
    \hat{\beta} &=
        \frac{\alpha \mathbb{P}(y=1,a=1)g_2 + \beta \mathbb{P}(y=1,a=0)(1-g_1)}{\mathbb{P}(y=1,a=1)g_2 + \mathbb{P}(y=1,a=0)(1-g_1)}.
        \end{split}
\end{align}

To see the above we have
\begin{align*}
    \mathbb{P} (f(x)=1 | y=1, \hat{a}=1) &= \frac{A+B}{C+D},
\end{align*}
where
\begin{align*}
    A &= \mathbb{P}(f(x)=1 | y=1, a=1,\hat{a}=1) \mathbb{P}(\hat{a}=1|y=1,a=1)\mathbb{P}(y=1,a=1),\\
    B &= \mathbb{P}(f(x)=1 | y=1, a=0,\hat{a}=1) \mathbb{P}(\hat{a}=1|y=1,a=0)\mathbb{P}(y=1,a=0),\\
    C &= \mathbb{P}(\hat{a}=1|y=1,a=1)\mathbb{P}(y=1,a=1),\\
    D &= \mathbb{P}(\hat{a}=1|y=1,a=0)\mathbb{P}(y=1,a=0).
\end{align*}
The above can be rewritten as
\begin{align*}
    \mathbb{P} (f(x)=1 | y=1, \hat{a}=1) 
    &= \frac{(1-g_2)\mathbb{P}(f(x)=1 | y=1, a=1,\hat{a}=1) \mathbb{P}(y=1,a=1)}{(1-g_2)\mathbb{P}(y=1,a=1)+ g_1\mathbb{P}(y=1,a=0)}\\
    &+\;\;\;\;
     \frac{g_1\mathbb{P}(f(x)=1 | y=1, a=0,\hat{a}=1)\mathbb{P}(y=1,a=0)}{(1-g_2)\mathbb{P}(y=1,a=1) + g_1\mathbb{P}(y=1,a=0)},
\end{align*}
with $g_1,g_2$ defined as in \eqref{eq:cond-error-def}. 

Under Assumption I we have 
\[\mathbb{P}(f(x)=1 | y=1, a=1,\hat{a}=1)=\alpha
\]
and
\[
\mathbb{P}(f(x)=1 | y=1, a=0,\hat{a}=1)=\beta,
\]
and hence
\begin{align*}
    \mathbb{P} (f(x)=1 | y=1, \hat{a}=1)
    &=\frac{(1-g_2)\alpha \mathbb{P}(y=1,a=1)+g_1\beta\mathbb{P}(y=1,a=0)}{(1-g_2)\mathbb{P}(y=1,a=1) + g_1\mathbb{P}(y=1,a=0)}.
\end{align*}
Similarly, we obtain
\begin{align*}
    \mathbb{P} (f(x)=1 | y=1, \hat{a}=0) 
    &=\frac{g_2 \alpha  \mathbb{P}(y=1,a=1)  + (1-g_1)\beta\mathbb{P}(y=1,a=0)}{g_2\mathbb{P}(y=1,a=1)+ (1-g_1)\mathbb{P}(y=1,a=0)}.
\end{align*}
It follows that
\begin{align*}
    |\hat{\alpha} - \hat{\beta}| &= |\mathbb{P}(f(x)=1 | y=1, \hat{a}=1) -\mathbb{P} (f(x)=1 | y=1, \hat{a}=0)|\\
    &=\frac{|1-g_1-g_2|}{\big(\frac s r (1-g_1) + g_2\big)\big(\frac r s (1-g_2) + g_1 \big)} \cdot |\alpha - \beta|, 
\end{align*}
where
\begin{align*}
    r = \mathbb{P}(y=1,a=1), \qquad
    s = \mathbb{P}(y=1,a=0).
\end{align*}
A simple calculation then shows that the true bias and the estimated bias satisfy 
\begin{align}
\label{eq:bias_correction_formula}
     |\hat{\alpha} - \hat{\beta}| &= \gamma |\alpha - \beta|,
\end{align}
where the {\em distortion factor} $\gamma$ is defined as
    \begin{align}\label{eq:def_E}
    \gamma = \frac{|1-g_1-g_2|}{\big(\frac s r (1-g_1) + g_2\big)\big(\frac r s (1-g_2) + g_1 \big)} 
    \end{align}
    with larger values of $\gamma$ corresponding to higher accuracy estimates of bias,
    and the quantities $r,s$ are defined as
    \begin{align*}
    r = \mathbb{P}(y=1,a=1), \qquad  s = \mathbb{P}(y=1,a=0).
\end{align*}
We refer to $r/s$ as the {\em ratio of base rates} and assume that it is bounded in $(0,G)$ for some finite $G$, else there will be no good way to estimate the bias from a finite sample. Notice that if $g_1+g_2$ equals one, then $\gamma$ is zero. In every other case 
it is not hard to see that $\gamma \in [0,1]$.
\begin{remark}
Under Assumption I, $|\hat{\alpha} - \hat{\beta}|$ is an \emph{underestimate} of the bias of $f$, and is therefore not an unbiased estimator. This is evident from our derivation, but may not be obvious for someone using empirical value as a stand-in for the true value of bias.
\end{remark}
Hence, even in the simplified setting of Assumption I, the distortion in the estimated bias depends on the conditional errors of $h$, as well as the ratio of base rates $r/s$. Therefore, to accurately estimate the true bias one not only needs the conditional errors of $h$ but also an estimate of 
the ratio of base rates. 
Under Assumption I an optimal attribute classifier should aim to maximize $\gamma$ as defined in \eqref{eq:def_E} (since $\gamma \in [0,1]$). Our next theorem quantifies the structure of the optimal attribute~classifier.
\begin{theorem}
\label{thm:opt-classifier}
Assume that the distribution $P$ is such that the ratio of base rates equals one, i.e., $r=s$.  Furthermore, denote an attribute classifier $h$ as a tuple $(g_1, g_2)$ where $g_1$ and $g_2$ denote the conditional errors as defined in \eqref{eq:cond-error-def}. Then, under a given error budget $U$, i.e., 
$\mathbb{P}(h(x)\neq a, y=1)=U$,
the only global maximizers of $\gamma$ as defined in \eqref{eq:def_E} are the attribute classifiers 
$(0,{U}/{r})$ and $(U/r,0)$ (if $U\leq r$) or $({U}/{r}-1,1)$ and $(1,{U}/{r}-1)$ (if $U\geq r$).
\end{theorem}
\begin{proof}
In case of $r=s$, $\gamma$ as defined in \eqref{eq:def_E} 
simplifies to 
\begin{align*}
\gamma = \frac{|1-g_1-g_2|}{1-(g_1-g_2)^2},
\end{align*}
and the constraint 
$\mathbb{P}(h(x)\neq a, y=1)=\mathbb{P}(a=0, y=1)g_1 + \mathbb{P}(a=1, y=1)g_2=U$ is equivalent to 
$g_1 + g_2=
\frac{U}{r}$.
\vspace{5mm}
 Our goal is to show that $\argmax_{(g_1,g_2)\in[0,1]^2:\, g_1 + g_2=
\frac{U}{r}} \Big( \gamma \Big)$ equals
\begin{align*}
\{(\max\{{U}/{r}-1,0\},\min\{{U}/{r},1\}),(\min\{{U}/{r},1\},\max\{{U}/{r}-1,0\})\}.
\end{align*} 

\vspace{2mm}
Writing $g_2$ in terms of $g_1$, we get that
$$
g_2 = \frac{U}{r}-g_1
$$
and 
$$
\gamma(g_1) = \frac{|1- \frac{U}{r}|}{1-(2g_1 - \frac{U}{r})^2}.
$$
The constraint $(g_1,g_2)\in[0,1]^2$ is 
equivalent to 
\[g_1\in[\max\{\frac{U}{r}-1,0\},\min\{\frac{U}{r},1\}].
\]
We assume that $\frac{U}{r}\leq 2$. Maximizing $\gamma(g_1)$ is equivalent to maximizing $(2g_1-\frac{U}{r})^2$, and the maximum of the latter function is attained both at $g_1=\max\{\frac{U}{r}-1,0\}$ and $g_1=\min\{\frac{U}{r},1\}$, which correspond to $g_2=\min\{\frac{U}{r},1\}$  and $g_2=\max\{\frac{U}{r}-1,0\}$, respectively.
\end{proof}


\begin{figure*}[t]
\centering
    \includegraphics[width=2in]{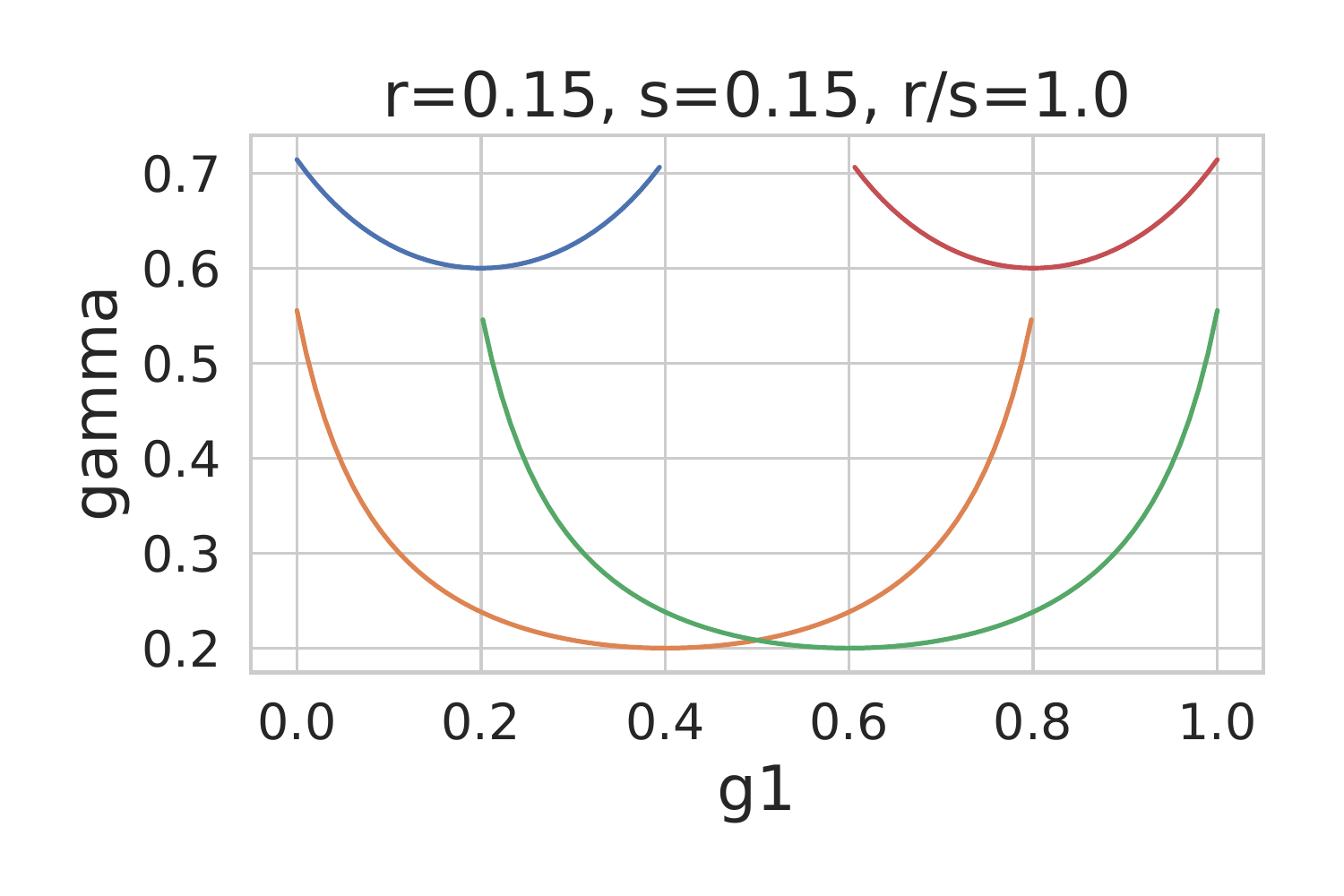}
    \includegraphics[width=2in]{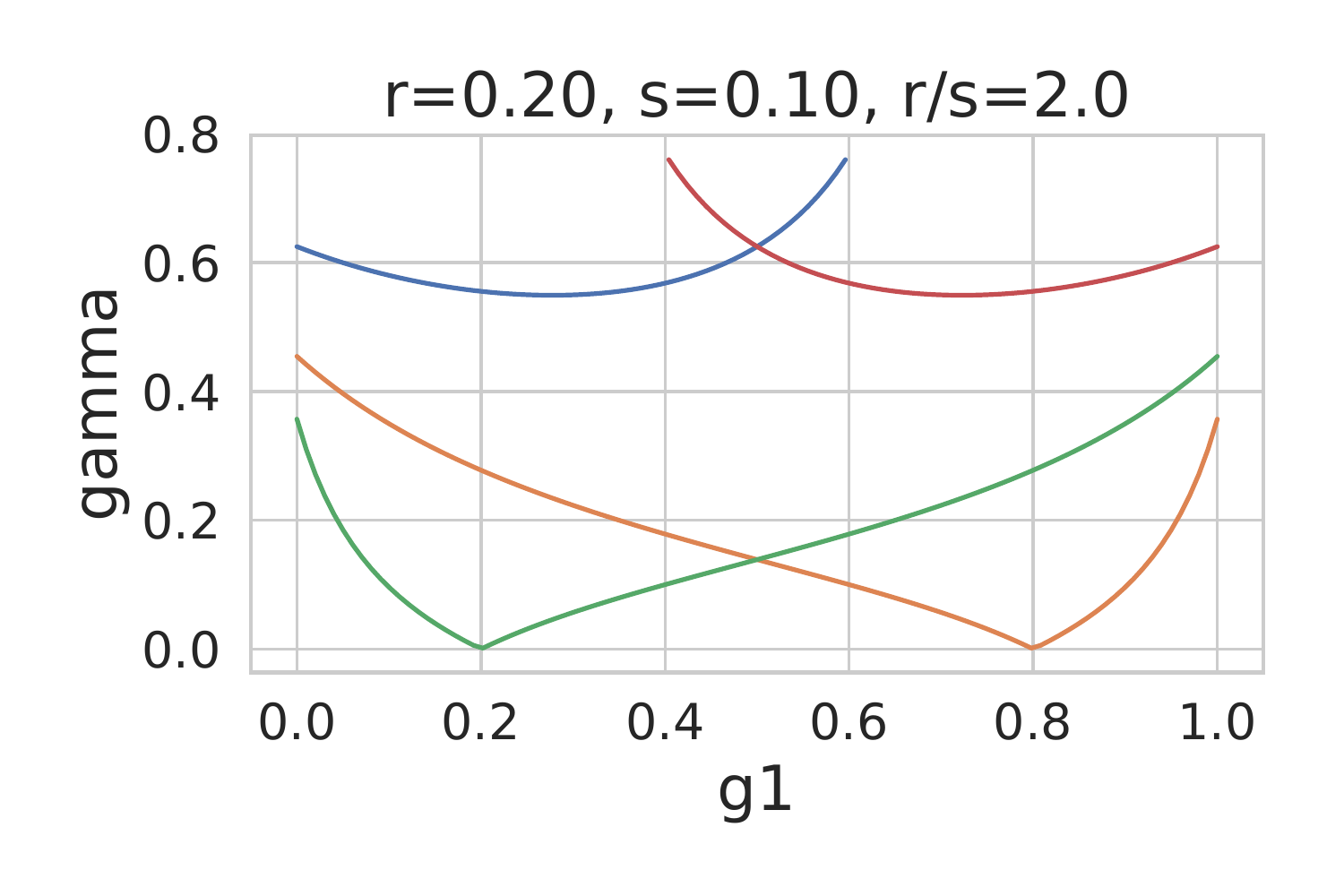}
    \includegraphics[width=2.65in]{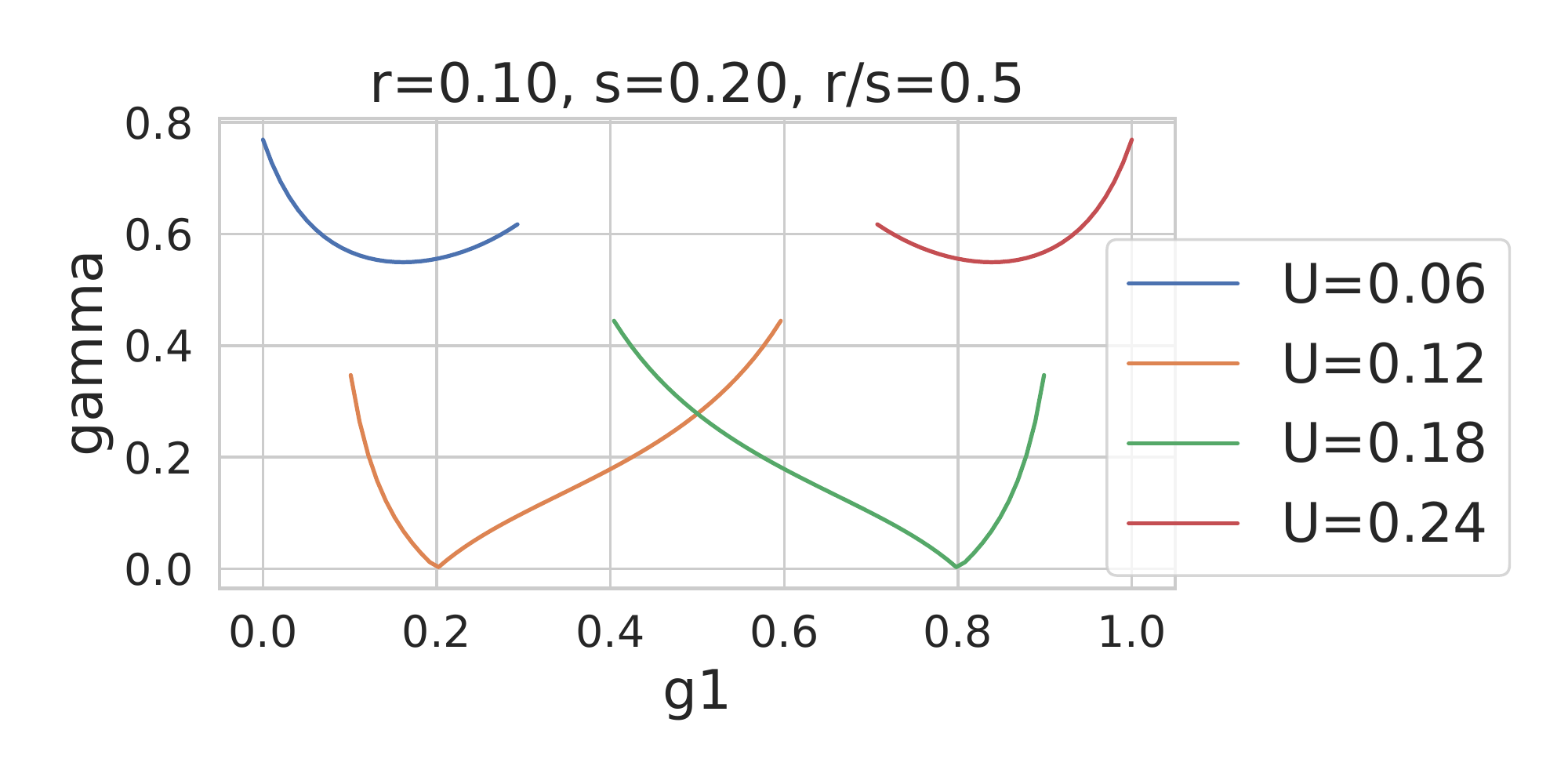}
    \vspace{-0.15in}
    \caption{The distortion factor~$\gamma$ as a function of $g_1$ for  different $r/s$ and under different error budgets $U=sg_1+rg_2$ for the attribute classifier. Apparently, $\gamma$ always attains its maximum either at the smallest or the largest possible value of $g_1$.}
    \label{fig:gamma_plot_w_base_rate}
\end{figure*}

\begin{figure*}[t]
     \centering
     \includegraphics[width=4.7in]{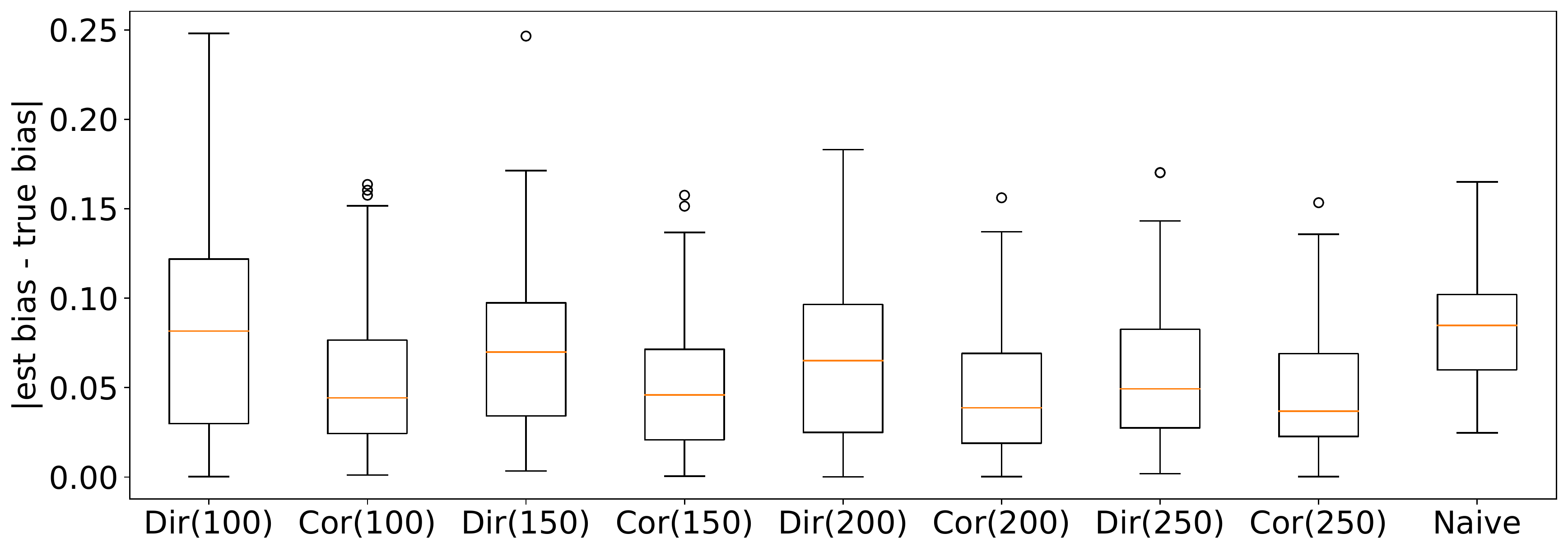}
     \caption{Experiment on the FIFA 20 player 
     dataset---independence 
     assumption~\eqref{eq:cond_ind_assumption} 
     holds. 
     Absolute difference between the true bias and the estimated bias for direct (Dir) and 
     corrected~(Cor) estimation, when  
     100 to 250 
     points of common data over $(y,a)$ are available, 
     and for naive (Naive) estimation. }
     \label{fig:exp_fifa_dataset}
 \end{figure*}

The above theorem implies that even under simplifying assumptions, the structure of the optimal attribute classifier is counter-intuitive: to estimate bias as accurately as possible, one may want to distribute the errors of the given attribute classifier as unevenly as possible! Hence, great care must be taken when designing the attribute classifier for the purpose  of bias estimation. Furthermore, while
we only consider the case $r=s$, 
Figure~\ref{fig:gamma_plot_w_base_rate} shows the distortion factor~$\gamma$  as a function of $g_1$ for a fixed error budget $U$ also when $r\neq s$ (middle and right plot). As we can see, also then 
$\gamma$ attains its maximum either at the smallest or the largest possible value of $g_1$, which corresponds to distributing the error among the two groups as unevenly as possible. 


\section{Perspective of the User of an Attribute Classifier}
\label{sec:experiments}


If the user of the attribute classifier expects Assumption~I as defined in Eq.~\eqref{eq:cond_ind_assumption} to (approximately) hold and she has some 
common data over $(y, a)$ 
from which she will  estimate the conditional errors~$g_1,g_2$ and the ratio of base rates~$r/s$, 
exploiting 
\eqref{eq:bias_correction_formula}, she can try to improve the naive bias estimate $|\hat{\alpha}-\hat{\beta}|$ by dividing it by an estimate of the distortion factor~$\gamma$. In this section, we present an experiment that illustrates that such an approach indeed yields a better estimate of the true bias~$|\alpha-\beta|$. 
In a data scarce regime, 
this approach 
also compares favorably to using the available data for directly estimating $|\alpha-\beta|$.

\vspace{1mm}
\textbf{Dataset.~~}
We use the FIFA~20 player dataset\footnote{\url{https://www.kaggle.com/stefanoleone992/fifa-20-complete-player-dataset}}.
We
predict whether a soccer player's wage is above ($y=+1$) or below ($y=-1$) the median wage based on the player's age and their \emph{Overall} attribute. For doing so we train a one-hidden-layer NN. We restrict the dataset to contain only players of English or German nationality (leaving us with 2883 players) and consider nationality as sensitive attribute. We train an LSTM \cite{lstm} to predict this attribute from a player's name. 
In this case we can expect the conditional independence assumption~\eqref{eq:cond_ind_assumption} to hold since given a player's wage and nationality, their name should be conditionally independent of age and \emph{Overall} attribute.
Indeed, the measure proposed in \cite{awasthi2019equalized} 
is small enough (see below) as to confirm that our expectation holds and the independence assumption is satisfied.

\begin{figure*}[t]
    \centering
    \includegraphics[width=5in]{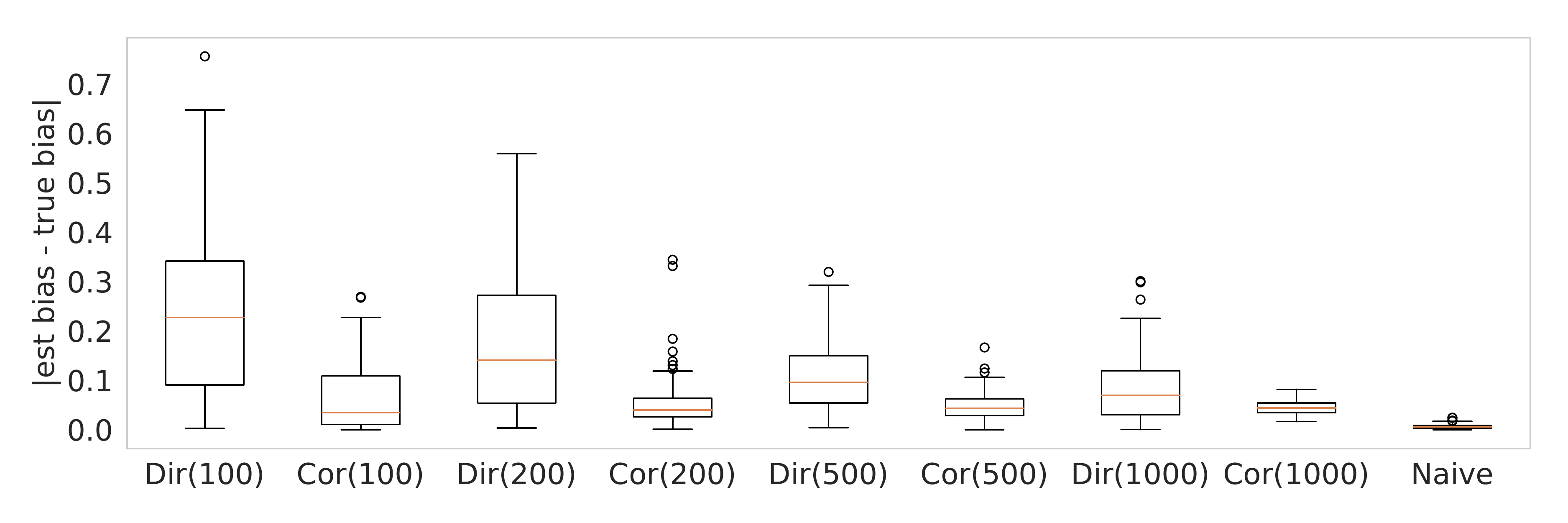}
     \caption{Experiment on the UCI Adult 
     dataset---independence 
     assumption~\eqref{eq:cond_ind_assumption} is violated.  Absolute difference between the true bias and the estimated bias for direct (Dir) and 
     corrected~(Cor) estimation, when 
     100 to 1000 
     points of common data over $(y,a)$ are available, 
     and for naive (Naive)~estimation.}
     \label{fig:adult_bias_est}
\end{figure*}

\vspace{1mm}
\textbf{Experimental setup.~~} We randomly split the dataset into three batches of sizes 1500, 1133 and 250, respectively. We use the first batch to train the label and the attribute classifier. On the second batch, we compute both the true bias~$|\alpha-\beta|$ and the \emph{naive estimate}~$|\hat{\alpha}-\hat{\beta}|$. Finally, we use $n\in\{100,150,200,250\}$ points from the third batch to estimate $g_1,g_2,r,s$, and thus $\gamma$, and also to directly estimate $|\alpha-\beta|$. We refer to the latter estimate as \emph{direct estimate}. Exploiting~\eqref{eq:bias_correction_formula}, we use the estimate of $\gamma$ to correct the naive estimate~$|\hat{\alpha}-\hat{\beta}|$ by dividing it by the estimate of $\gamma$. We refer to the 
result 
as \emph{corrected estimate}. 

Figure~\ref{fig:exp_fifa_dataset} shows the absolute difference between the true bias and the various estimates, where for the direct estimate and the corrected estimate we show the results depending on $n$, the amount of data for which we have both $y$ and $a$. The boxplots are obtained from running the experiment  100 times. 
We can see that the corrected estimate significantly improves over the naive estimate (e.g., the mean absolute difference, averaged over the 100 runs, is 0.0831 for the naive estimate and 0.0461 for the corrected estimate 
with 
$n=250$). 
Furthermore, the corrected estimate consistently outperforms the direct estimate. When 
only $n=100$ points of common data over $(y,a)$ 
are used, 
the mean absolute difference for the corrected estimate is 0.054 while for the direct estimate it is 0.0837. 
%
In this experiment, the mean true bias 
is 0.14, the mean error of the label and the attribute classifier is 0.18 and 0.25, respectively, and the mean violation of the independence assumption according to the measure proposed in \cite{awasthi2019equalized} is $0.017$.

\subsection{An Active Sampling based Algorithm for Bias Estimation in General Settings}
\label{sec:active_sampling}
In some real-world scenarios the conditional independence assumption might be violated, and we would like to ask, in those cases, does the correction as specified in Eq.~\eqref{eq:bias_correction_formula} still give a good estimate of the true bias? 
We show in the following that the correction might not always lead to a better bias estimate when Assumption I does not hold.
Therefore, in order to obtain a good bias estimate in general, we explore active-sampling strategies that aim to use as little common data over $(y,a)$ as possible, ideally comparable to the setting when the independence assumption holds.

\vspace{1mm}
\textbf{Dataset.~~} We use the UCI Adult dataset\footnote{\url{https://archive.ics.uci.edu/ml/datasets/Adult}}, which comes divided into a training and a test set.
On the training set, we train Random Forest classifiers for both predicting the label (income: $>50K$ or not, using all features present in the dataset except gender) and the attribute (gender: ``Male'' or ``Female'' as categorized in the dataset, using all features except income). 
We used the measure proposed in \cite{awasthi2019equalized} to evaluate 
whether the independence assumption~\eqref{eq:cond_ind_assumption} (approximately) holds and found that 
the independence assumption clearly fails to hold. 

\paragraph{Inaccurate bias estimation when the conditional independence assumption is violated.}
We use the UCI Adult training set for training our label classifier and attribute classifier. For the Adult test set (around 16,000 examples), we hold out a set with $2,000$ examples for estimating $g_1, g_2, r, s$ since it requires some common data over $(y, a)$, and use the rest for estimating the true bias ($|\alpha-\beta|$).
In Figure~\ref{fig:adult_bias_est}, we show the absolute difference between the estimated bias and the true bias.
For the estimated bias, we use the same setting as the previous experiment, with direct estimate, corrected estimate, and naive estimate.
We can see the the correction gives a more accurate estimate than direct estimation, however, the naive estimate remains to be the one with the most accurate estimate.

\paragraph{Active sampling algorithm.}
\label{sec:active}
We next address the limitations of the approaches explored above.
In particular, we propose a \emph{general active-sampling} based strategy that the label classifier can use to accurately estimate the bias, using as little common data over $(y,a)$ as possible, given the attribute classifier. Furthermore, we will make \emph{no assumption} about the conditional independence or the structure of the attribute classifier.

In the general setting the estimated bias and the true bias are related as 
follows:
\begin{theorem}
\label{thm:opt-classifier-general}
The true bias $\alpha-\beta$ can be derived from $\hat{\alpha}, \hat{\beta}$ as:
\begin{align}
\alpha-\beta 
&=\frac{\hat{\alpha}(\frac{s}{r}g_1+1-g_2)(1-\delta_1+\frac{r}{s}\delta_2)}{1-\delta_1-\delta_2} \nonumber \\ 
&- 
\frac{\hat{\beta}(1-g_1+\frac{r}{s}g_2)(1+\frac{s}{r}\delta_1-\delta_2)}{1-\delta_1-\delta_2},
\label{eq:true-bias-updates}
\end{align}
where 
\begin{align}
\label{eq:delta-def}
    \delta_1 &= \mathbb{P}(\hat{a}=1 | f(x)=1, a=0, y=1), \nonumber \\
    \delta_2 &= \mathbb{P}(\hat{a}=0 | f(x)=1, a=1, y=1),
\end{align}
and $g_1, g_2$ are as defined in Eq.~\eqref{eq:cond-error-def}. 
\end{theorem}

\begin{figure*}[t]
\centering
\includegraphics[width=2.7in]{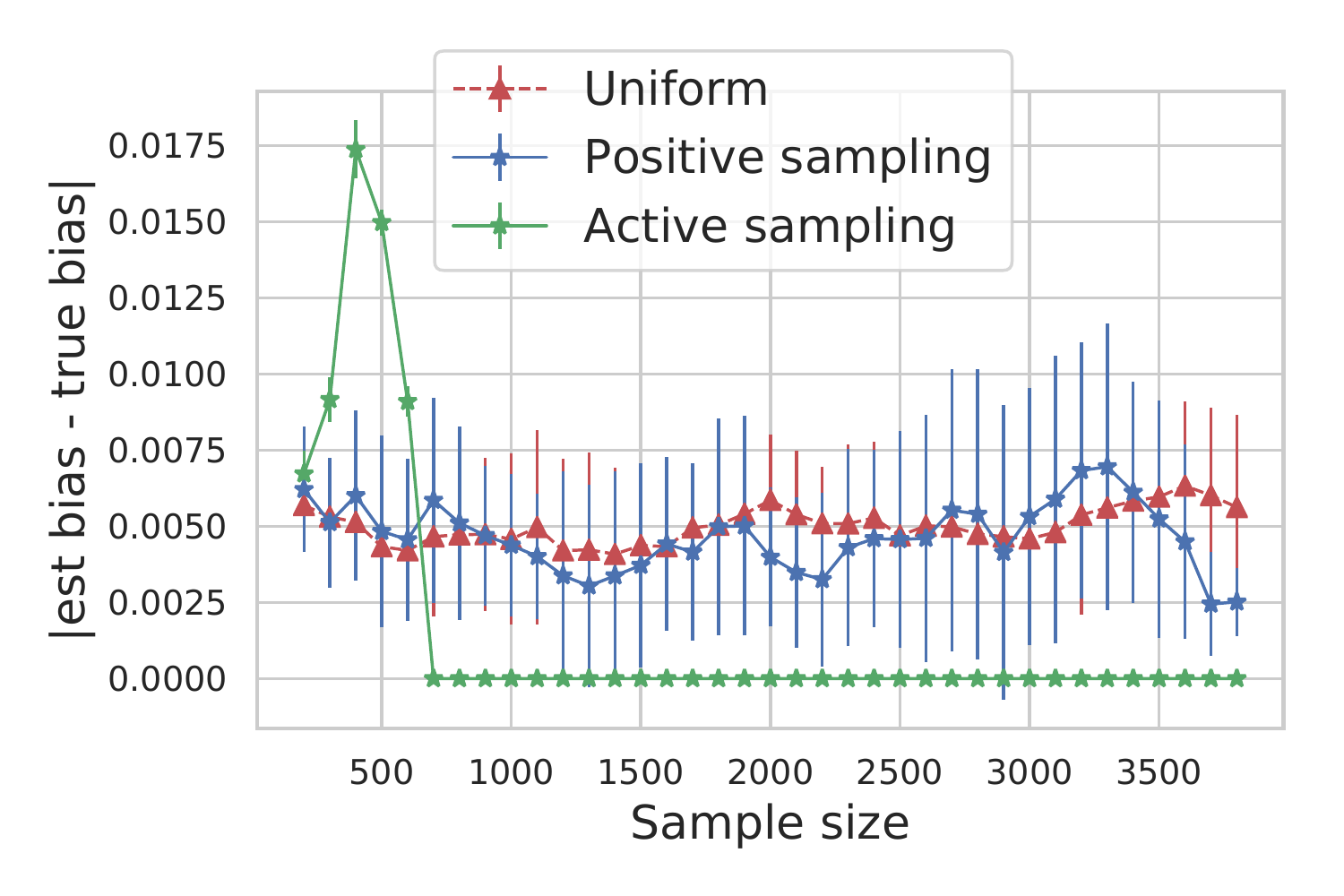}
\hspace{1.2cm}
\includegraphics[width=2.65in]{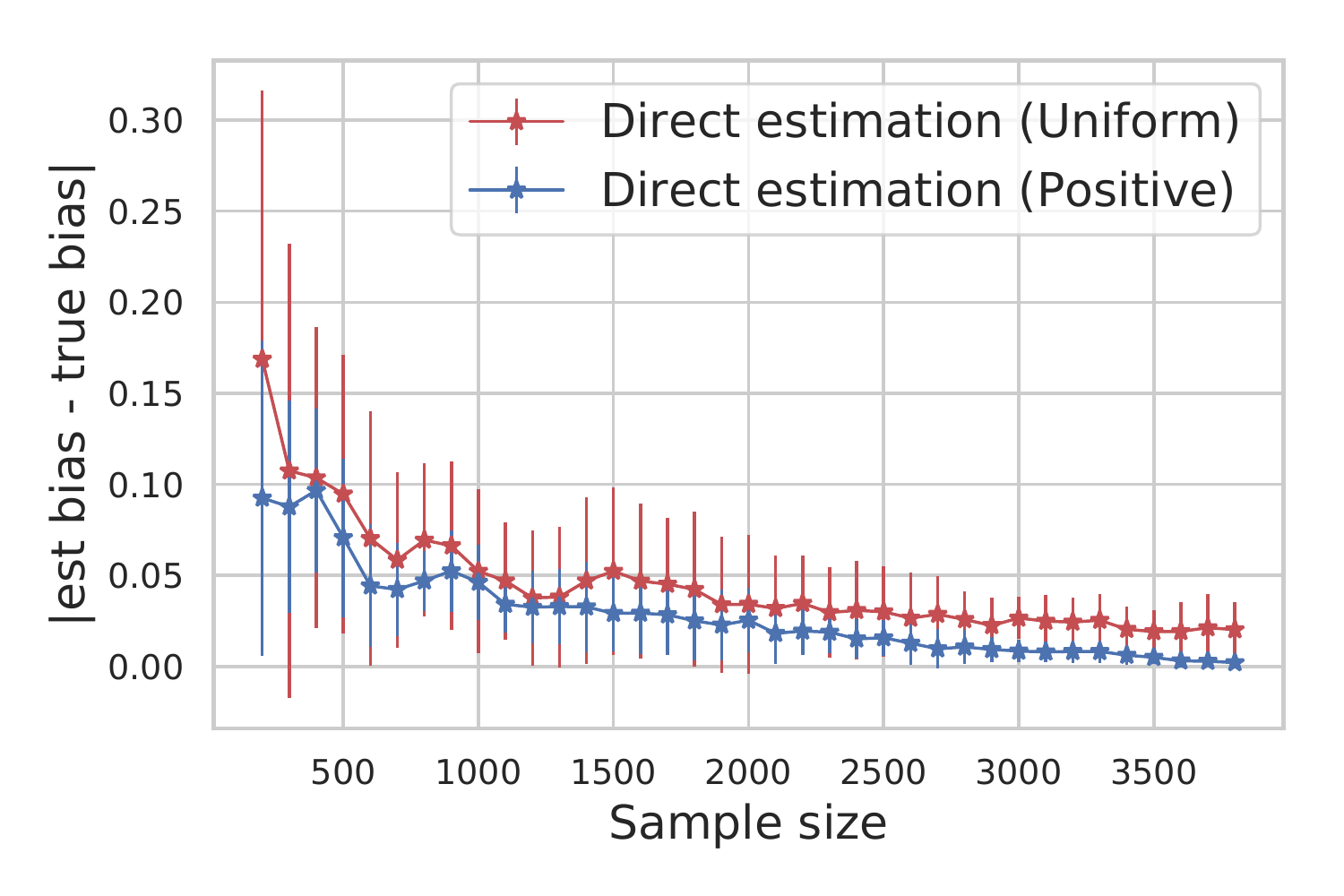}
    \caption{The absolute difference between estimated bias and the true bias, based on different \textbf{active sampling} strategies (left), and based on \textbf{direct estimation} over selected samples using uniform/positive sampling (right). Both are results averaged over $10$ runs with random initialization.}
    \label{fig:active_learning_bias_est}
\end{figure*}

\begin{algorithm}
\caption{Active Sampling}
\begin{algorithmic}[1]
\item[] \textbf{Input:} $D_2 = \{(x_1, y_1), \dots, (x_m, y_m)\}$, attribute classifier $h$, parameters $b, w$.
\item [] \textbf{Output:} Estimated bias $|\hat{\alpha} - \hat{\beta}|$, tolerance $\epsilon$.
\STATE Sample a set of $b$ examples uniformly at random from $D_2$ conditioned on $y=1$ and get their true attributes. Get estimates $\hat{r}$ and $\hat{s}$ using the sampled examples.
\STATE Initialize $t=0$ and $g^{(0)}_1, g^{(0)}_2, \delta^{(0)}_1, \delta^{(0)}_2$ to be zero.
\FOR{$t=1, \dots$}
\STATE Sample a batch of $b$ unlabeled~(i.e. no sensitive attribute information) examples uniformly at random from $D_2$ conditioned on $y=1$.
\STATE Sort the examples in ascending order according to the uncertainty of $h$, i.e., $|h(x)-0.5|$.
\STATE Obtain true attribute information the first $w$ examples and append to the labeled set of data. Compute $g^{(t+1)}_1, g^{(t+1)}_2, \delta^{(t+1)}_1, \delta^{(t+1)}_2$ on the labeled set using \eqref{eq:cond-error-def} and \eqref{eq:delta-def}. 
\STATE If $|g_1^{(t+1)} - g_1^{(t)}|, |g_2^{(t+1)} - g_2^{(t)}|, |\delta_1^{(t+1)} - \delta_1^{(t)}|, |g\delta_2^{(t+1)} - \delta_2^{(t)}|$ are all bounded by $\epsilon$ then \textbf{break}.
\ENDFOR
\STATE Output estimated bias obtained by using \eqref{eq:true-bias-updates} and the above estimates.
\end{algorithmic}
\label{alg:active-sampling}
\end{algorithm}

While $r,s$ can be estimated from a small amount of data, we explore sampling strategies for estimating the attribute classifier dependent quantities, namely, $g_1, g_2$ and $\delta_1, \delta_2$. Two natural sampling schemes are: 
(i) {\em Uniform sampling} where we uniformly sample a set of examples and get their true sensitive attribute information, and (ii) {\em Positive sampling}: where we perform uniform sampling on the subset of examples with $y=1$. We expect positive sampling to perform better than uniform sampling, since the quantities that we are estimating are conditioned on $y=1$. We compare the above two approaches with (iii) an {\em active sampling} approach that involves querying for true sensitive attributes~(conditioned on $y=1$) of examples on which the attribute classifier is the most uncertain (as described in Algorithm~\ref{alg:active-sampling}).
For each of the methods, in each iteration, we compute the estimated bias on the UCI Adult test set by using true values of $a$ over the selected samples, and the predicted values $\hat{a}$ from the attribute classifier for the unselected samples, using a sampling batch size of $100$. 
The results are shown in Figure~\ref{fig:active_learning_bias_est} (left), where we see that the uncertainty based active sampling approach requires significantly less common data to accurately estimate the~bias.

Furthermore, we also compare our approach with (iv) {\em Direct estimation}: where we directly estimate the bias~($|\alpha-\beta|$) based on the currently sampled set $S$, i.e., the small amount of common $(y, a)$ on the same test set, using a uniform sampling strategy and a positive sampling strategy. The result is shown in Figure~\ref{fig:active_learning_bias_est} (right). Note that the magnitude in the y-axis is much larger compared to the left figure, indicating that direct estimation using few samples tends to give less reliable estimation of the bias.




\section{Discussion and Conclusion}
We formalized and studied a commonly occurring scenario in fairness evaluation and auditing namely, the lack of common data involving both the class labels and sensitive attributes. Our experiments and theoretical analysis reveal that great care must be taken in designing attribute classifiers for the purpose of bias estimation. Furthermore, in certain scenarios, the structure of the optimal attribute classifier can be contradictory to natural criteria such as accuracy and fairness. From the perspective of the designer of the attribute classifier, maximizing the distortion factor as in \eqref{eq:def_E} is a challenging optimization problem. It would be interesting to explore efficient algorithms for solving this. Throughout our analysis we assume that the datasets $D_1, D_2$ are sampled from the same distribution. It would be interesting to extend our theory to the more realistic case of when the dataset distributions are different. Finally, it would also be interesting to provide theoretical guarantees for the active sampling scheme in Algorithm~\ref{alg:active-sampling}.





\bibliography{faact_arxiv}
\bibliographystyle{plainnat}


\appendix
\section{Appendix}

\begin{proof}[Proof of Theorem~\ref{thm:bayes-opt-example}]
Let the joint distribution $Q$ be as follows:
\begin{align*}
    \mathbb{P}(x_1=1, x_2=0, a=1, y=1) &= \frac{1}{6},\\
    \mathbb{P}(x_1=1, x_2=1, a=1, y=1) &= \frac{1}{6},\\
    \mathbb{P}(x_1=1, x_2=1, a=0, y=1) &= \frac{1}{6},\\
    \mathbb{P}(x_1=1, x_2=0, a=0, y=1) &= \frac{1}{6},\\
    \mathbb{P}(x_1=1, x_2=1, a=1, y=0) &= \frac{1}{6},\\
    \mathbb{P}(x_1=1, x_2=0, a=0, y=0) &= \frac{1}{6}.
\end{align*}
Then we have for $f=x_2$, 
\begin{align*}
    \mathbb{P}(f=1 | y=1, a=0) &= \mathbb{P}(x_2=1 | y=1, a=0)= \frac{1}{2}.
\end{align*}
Similarly,
\begin{align*}
    \mathbb{P}(f=1 | y=1, a=1) &= \mathbb{P}(x_2=1 | y=1, a=1)= \frac{1}{2}.
\end{align*}
Hence the true bias of $f$ is zero. Next consider the Bayes optimal predictor for $a$ that makes predictions based on $\mathbb{P}(a | x_1, x_2)$. If this probability is half, then we consider an arbitrary predictor for the Bayes optimal classifier. In this case the Bayes optimal predictions $\hat{a}$ are shown in the table below.
\begin{table}[htbp]
\centering
\begin{tabular}{ |c|c|c|c|c| }
\hline
$x_1$ & $x_2$ & $a$ & $y$ & $\hat{a}$ \\ 
\hline
$1$ & $0$ & $1$ & $1$ & $0$\\
\hline
$1$ & $1$ & $1$ & $1$ & $1$\\
\hline
$1$ & $1$ & $0$ & $1$ & $1$\\
\hline
$1$ & $0$ & $0$ & $1$ & $0$\\
\hline
$1$ & $1$ & $1$ & $0$ & $1$\\
\hline
$1$ & $0$ & $0$ & $0$ & $0$\\
\hline
\end{tabular}
\end{table}
Next using $\hat{a}$ instead of $a$ to estimate the bias of $f$ we get
\begin{align*}
    \mathbb{P}(f=1 | y=1, \hat{a}=0) &= \mathbb{P}(x_2=1 | y=1, \hat{a}=0)= 0.
\end{align*}
Similarly,
\begin{align*}
    \mathbb{P}(f=1 | y=1, \hat{a}=1) &= \mathbb{P}(x_2=1 | y=1, \hat{a}=1)= 1.
\end{align*}
Hence using a Bayes optimal attribute classifier leads to an estimated bias of one whereas the true bias is zero. In this case, using a random attribute 
predictor, i.e. predicting the class label via a coin toss is better than using the Bayes optimal~one!
\end{proof}

\begin{proof}[Proof of Theorem~\ref{thm:opt-classifier-general}]
To see \eqref{eq:true-bias-updates} notice that from the definition of $\hat{\alpha}$, we have
\begin{align}
\label{eq:intermediate}
    \hat{\alpha} &= \mathbb{P}(f(x)=1 | y=1, \hat{a}=1) = \frac{\mathbb{P}(f(x)=1, y=1, \hat{a}=1)}{\mathbb{P}(y=1, \hat{a}=1)}\nonumber\\ 
    &= \frac{\mathbb{P}(a=0, f(x)=1, y=1, \hat{a}=1) + \mathbb{P}(a=1, f(x)=1, y=1, \hat{a}=1)}
    {\mathbb{P}(y=1, a=0, \hat{a}=1) + \mathbb{P}(y=1, a=1, \hat{a}=1)}
    \end{align}
Next we first simplify the two terms in the numerator. We can write
\begin{align*}
    &\mathbb{P}(a=0, f(x)=1, y=1, \hat{a}=1) \\
    &= \mathbb{P}(\hat{a}=1|f(x)=1, a=0, y=1)\mathbb{P}(f(x)=1|a=0,y=1)\mathbb{P}(a=0,y=1)\\
    &= \delta_1 \beta s.
\end{align*}
Similarly, we write
\begin{align*}
    &\mathbb{P}(a=1, f(x)=1, y=1, \hat{a}=1) \\
    &= \mathbb{P}(\hat{a}=1|f(x)=1, a=1, y=1) \mathbb{P}(f(x)=1|a=1,y=1)\mathbb{P}(a=1,y=1)\\
    &= (1-\delta_2) \alpha r.
\end{align*}
For the denominator we have
\begin{align*}
    \mathbb{P}(y=1, a=0, \hat{a}=1) &= \mathbb{P}(\hat{a}=1| y=1, a=0)P(y=1, a=0)\\
    &= g_1 s.
\end{align*}
Similarly, we have
\begin{align*}
    \mathbb{P}(y=1, a=1, \hat{a}=1) &= \mathbb{P}(\hat{a}=1| y=1, a=1)P(y=1, a=1)\\
    &= (1-g_2) r.
\end{align*}
Substituting into \eqref{eq:intermediate} we get
\begin{align*}
    \hat{\alpha} &= \frac{\delta_1 \beta s + (1-\delta_2)\alpha r}{g_1 s + (1-g_2) r}.
\end{align*}
Similarly for $\hat{\beta}$, it is easy to derive:
\begin{align*}
    \hat{\beta} &= \mathbb{P}(f=1 | y=1, \hat{a}=0) = \frac{\delta_2 \alpha r + (1-\delta_1)\beta s}{(1-g_1) s + g_2 r}.
\end{align*}
Hence we have:
\begin{align}
    \hat{\alpha}[g_1 \frac{s}{r}  + (1-g_2)] = \delta_1  \frac{s}{r} \beta + (1-\delta_2) \alpha,
\label{eq_alpha}
\end{align}
\vspace{-0.1in}
\begin{align}
    \hat{\beta}[(1-g_1) + g_2 \frac{r}{s}] = \delta_2 \frac{r}{s} \alpha + (1-\delta_1)\beta.
\label{eq_beta}
\end{align}

Multiplying Eq.~\eqref{eq_alpha} by $(1-\delta_1+\frac{r}{s}\delta_2)$, multiplying Eq.~\eqref{eq_beta} by $(1+\frac{s}{r}\delta_1-\delta_2)$, and subtracting the latter from the former, we get:
\begin{align*}
&(1-\delta_1-\delta_2)(\alpha-\beta) =\\
&\hat{\alpha}[g_1 \frac{s}{r}  + (1-g_2)](1-\delta_1+\frac{r}{s}\delta_2) - \hat{\beta}[(1-g_1) + g_2 \frac{r}{s}](1+\frac{s}{r}\delta_1-\delta_2).
\end{align*}
Dividing the right hand side by $1-\delta_1-\delta_2$ we have Eq.~\eqref{eq:true-bias-updates}.
\end{proof}

\end{document}